\pdfoutput=1
\documentclass[pdflatex,sn-nature]{sn-jnl}%

\usepackage{graphicx}%
\usepackage{multirow}%
\usepackage{amsmath,amssymb,amsfonts}%
\usepackage{amsthm}%
\usepackage[title]{appendix}%
\usepackage{xcolor}%
\usepackage{textcomp}%
\usepackage{manyfoot}%
\usepackage{booktabs}%
\usepackage{algorithm}%
\usepackage{algorithmicx}%
\usepackage{algpseudocode}%
\usepackage{listings}%
\usepackage{soul}

\usepackage{mathtools}
\usepackage{bookmark}         %
\usepackage[capitalize,noabbrev]{cleveref}   %

\makeatletter
\let\sn@orig@float@endH\float@endH
\def\float@endH{%
  \def\@tempa{table}%
  \ifx\@captype\@tempa
    \end{threeparttable}\end{center}%
    \sn@orig@float@endH
    \endgroup
  \else
    \sn@orig@float@endH
  \fi}
\makeatother

\theoremstyle{thmstyleone}%
\theoremstyle{thmstylethree}%
\theoremstyle{thmstyletwo}%

\crefname{equation}{equation}{equations}%
\Crefname{equation}{Equation}{Equations}

\crefname{section}{Sec.}{Secs.}
\Crefname{section}{Section}{Sections}

\crefname{proposition}{Prop.}{Props.}
\Crefname{proposition}{Proposition}{Propositions}

\crefname{algorithm}{Supplementary Algorithm}{Supplementary Algorithms}
\Crefname{algorithm}{Supplementary Algorithm}{Supplementary Algorithms}

\crefname{figure}{Fig.}{Figs.}%
\Crefname{figure}{Figure}{Figures}

\crefname{appendix}{Supplementary Note}{Supplementary Notes}
\Crefname{appendix}{Supplementary Note}{Supplementary Notes}

\crefname{table}{Supplementary Table}{Supplementary Tables}%
\Crefname{table}{Supplementary Table}{Supplementary Tables}

\crefname{suppfigure}{Supplementary Fig.}{Supplementary Figs.}
\Crefname{suppfigure}{Supplementary Fig.}{Supplementary Figs.}

\crefname{theorem}{Thm.}{Thms.}
\Crefname{theorem}{Theorem}{Theorems}

\crefname{corollary}{Cor.}{Cors.}
\Crefname{corollary}{Corollary}{Corollaries}

\begin{document}

\title{More Bang for the Buck: Improving the Inference of Large Language Models at a Fixed Budget using Reset and Discard (ReD)}

\author[1,2]{\fnm{Sagi} \sur{Meir}}

\author[1,2]{\fnm{Tommer D.} \sur{Keidar}}

\author[3]{\fnm{Noam} \sur{Levi}}

\author[1,2,4]{\fnm{Shlomi} \sur{Reuveni}}

\author*[1,2,4]{\fnm{Barak} \sur{Hirshberg}}\email{hirshb@tauex.tau.ac.il}

\affil[1]{\orgdiv{School of Chemistry}, \orgname{Tel Aviv University}, \orgaddress{\city{Tel Aviv}, \postcode{6997801}, \country{Israel}}}

\affil[2]{\orgdiv{The Center for Physics and Chemistry of Living Systems}, \orgname{Tel Aviv University}, \orgaddress{\city{Tel Aviv}, \postcode{6997801}, \country{Israel}}}

\affil[3]{\orgdiv{School of Physics and Astronomy}, \orgname{Tel Aviv University}, \orgaddress{\city{Tel Aviv}, \postcode{6997801}, \country{Israel}}}

\affil[4]{\orgdiv{The Center for Computational Molecular and Materials Science}, \orgname{Tel Aviv University}, \orgaddress{\city{Tel Aviv}, \postcode{6997801}, \country{Israel}}}

\abstract{  The performance of large language models (LLMs) on verifiable tasks is usually measured by pass@$k$, the probability of solving a problem correctly at least once in $k$ trials.
    At a fixed budget, a more suitable metric is coverage@cost, the average number of unique problems solved as a function of the total number of attempts.
    We connect the two metrics and show that the empirically-observed power-law behavior in pass@$k$ leads to a sublinear growth of the coverage@cost (diminishing returns).
    To solve this problem, we propose Reset-and-Discard (ReD), a query method of LLMs that increases coverage@cost for a given budget, regardless of the pass@$k$ form. 
    Moreover, given a pass@$k$, we can quantitatively predict the savings in the total number of attempts using ReD. If pass@$k$ is not available for the model, ReD can infer its power-law exponent.
    {Experiments across coding (HumanEval), math (GSM8K), and reasoning (MMLU-Pro) benchmarks, as well as real GitHub issues solved by a coding agent (SWE-bench Lite), demonstrate that ReD substantially reduces the required attempts, tokens, and USD cost to reach a desired coverage, while also offering an efficient way to measure inference power-laws. ReD's advantage is maintained for imperfect verifiers.
    }}

\maketitle

\section{Introduction}
\label{sec:intro}

Large Language Models (LLMs) have effectively revolutionized every aspect of our day-to-day lives, and especially tasks such as coding, math, reasoning and formal proof search. Such tasks are often verifiable, i.e., the correctness of an answer can be checked automatically by a program, known as a verifier (for example, unit tests for generated code). A standard way to improve LLM performance on verifiable tasks is to generate multiple independent candidate solutions.
Across model families and benchmarks, this can yield striking gains: 
relatively small models can approach or even exceed the single-sample performance of larger models when allowed many attempts~\citep{chen2021evaluatinglargelanguagemodels,brown2024llmonkeys,snell2024computeoptimal,wu2024inference_scaling}.
These gains are typically reported via \text{pass@}$k$: the probability that at least one out of $k$ generations solves a \emph{single}, randomly-chosen instance~\citep{kulal2019spoc,chen2021evaluatinglargelanguagemodels,first2023baldurwholeproofgenerationrepair,brown2024llmonkeys,hassid2024larger,hughes2024bestofnjailbreaking,chen2024llmcallsneedscaling,levi2024simple, schaeffer2025large,ehrlich2025codemonkeysscalingtesttimecompute,kwok2025robomonkeyscalingtesttimesampling}.
Recent work has uncovered a remarkably regular empirical pattern: for many tasks, $1-\text{pass@}k$ decays as a power-law in $k$ across orders of magnitude~\citep{schaeffer2025large,levi2024simple,kazdan2025efficient}.

In practice, inference compute is often shared across a workload of problems: Coding agents are judged by the fraction of real GitHub issues they resolve~\citep{jimenez2024swebenchlanguagemodelsresolve}. Safety tests try to break a model's safeguards on as many harmful requests as possible~\citep{hughes2024bestofnjailbreaking}, and training sets are built by sampling a model on many problems and keeping its correct solutions~\citep{zelikman2022star,gulcehre2023rest,bansal2025smaller}. Similar settings arise in software verification~\citep{first2023baldurwholeproofgenerationrepair} and in services that answer many user queries for a fee~\citep{chen2024frugalgpt}.
In these settings, the relevant objective is not $\text{pass@}k$, i.e., the probability of succeeding on a single problem on average, but \emph{how many distinct problems can be solved under a fixed global budget out of the entire pool}.
To formalize this objective we adopt \textbf{coverage@cost}, defined as the expected number of unique problems solved after spending a total budget (measured in attempts, tokens or USD).
coverage@cost is a metric that rewards solving \emph{more} problems rather than spending disproportionate compute on a small number of hard instances.
Crucially, coverage@cost depends not only on the model, but also on the \emph{policy} used to allocate attempts across problems.

A straightforward allocation policy is \emph{solve-to-completion}: repeatedly sample a problem until it is solved, then move to the next.
Solve-to-completion (STC), and many other methods for spending inference compute, such as picking the best of many answers or searching over candidate solutions with a verifier, treat one problem at a time~\citep{snell2024computeoptimal,welleck24tmlr}. Even budget-aware evaluations often set a separate budget for each problem~\citep{wang2024tokeneconomies}.
When per-problem success probabilities vary (difficulty distribution), such policies are difficulty-biased, since hard problems consume many attempts, reducing throughput for the overall workload.
Under the empirically observed power-law $\text{pass@}k$ behavior, $1-\text{pass@}k \propto k^{-\alpha}$, commonly with $0<\alpha<1$, this effect can be extreme.
In that case, we show that STC yields \emph{sublinear} coverage@cost growth, meaning that additional budget buys diminishing returns in the number of distinct solved problems.
Therefore, efficient \emph{allocation} policies under a fixed budget across \emph{many} problems, whose difficulty is not known in advance, are needed.

We propose an alternative allocation policy, based on stochastic resetting, a principle that is well established outside of LLM inference.
Restart strategies (stopping a process and starting it anew) are commonly used in randomized algorithms and combinatorial search, to reduce per‑instance runtime and its variability~\citep{luby1993lasvegas,gomes1997heavytails,gomes2000heavytails,walsh1999smallworld,williams2003backdoors,streeter2007restartschedules}.
A parallel line of work characterizes when and how restart reduces mean first-passage times (the time at which a process first reaches a target state) in stochastic processes and non-equilibrium statistical physics~\citep{evans2011resetting,reuveni2016restart,pal2017restart,evans2020stochastic,de2022optimal,kumar2023universal}.
Recent developments have considerably expanded the use of resetting, demonstrating that it can accelerate sampling in high-dimensional diffusion and rare-event processes~\citep{blumer2022stochastic, blumer2024short, church2025accelerating, keidar2025adaptive, blumer2025have}.
In machine-learning training, restart ideas also appear in warm-restart learning-rate schedules, resetting-based regularization under label noise, and training speedup predictions with respect to a resetting time~\citep{loshchilov2016sgdr,bae2025stochastic,meir2025first}.
These examples hint that resetting might also be advantageous in LLM inference. However, it is not clear what resetting would mean in this context, how to do it, and whether it would provably lead to any acceleration.

To this end, we introduce \textbf{Reset-and-Discard (ReD)}:
(1) \emph{Reset} sampling to the next problem in line, after a fixed number of attempts $\tau$ (in the extreme, after every attempt), and
(2) \emph{Discard} problems as soon as they are solved so no further compute is spent on them.
Intuitively, failures provide information: a problem that has already failed several times is more likely to be hard, so the expected marginal gain from another attempt is smaller than starting fresh on a new problem.
ReD implements this ``breadth-first'' approach directly, and is straightforward to deploy in any pipeline that already supports repeated sampling and verification.

Our work proves the following: (1) We connect $\text{pass@}k$ to coverage@cost via renewal theory \cite{grimmett2020probability,Cox1967} and derive expressions for coverage@cost as a function of the budget, given pass@$k$.
(2) We prove that resetting strictly improves coverage@cost for {a given} budget and any underlying difficulty distribution, and that the optimal resetting protocol is ReD every attempt ($\tau=1$).
(3) Given an empirical or modeled $\text{pass@}k$ curve, we can predict the budget savings induced by ReD to reach any target coverage without running any inference under resetting.
(4) ReD trajectories provide a statistically efficient way to estimate inference power-law exponents when direct large-$k$ $\text{pass@}k$ measurement is impractical.

We validate these results on multiple LLMs and the HumanEval benchmark~\cite{chen2021evaluatinglargelanguagemodels} (164 Python programming problems checked by unit tests), the GSM8K dataset~\cite{cobbe2021trainingverifierssolvemath} (grade-school math word problems), a subset of the MMLU-Pro~\cite{wang2024mmlu} dataset of multiple-choice questions across academic disciplines (first 500 questions), and SWE-bench Lite~\cite{jimenez2024swebenchlanguagemodelsresolve} (real GitHub issues solved by a coding agent). These examples demonstrate that ReD increases the number of distinct problems solved at fixed budgets, yielding large savings relative to STC allocation.

The paper is structured as follows: in \cref{sec:illustrative_ex}, we first give an illustrative example. We explain what verifiable tasks are, how STC works, why it is inefficient, and what ReD does differently to provably improve the coverage. The following sections give the rigorous derivations of the theoretical results. \Cref{sec:map} connects pass@$k$ to coverage@cost of STC. \Cref{sec:red_improves} proves that ReD improves coverage@cost at every budget and that resetting after every attempt is optimal. \Cref{sec:inference_exponent} shows how to efficiently estimate the power-law exponent of pass@$k$ from ReD. Then, we verify the theoretical results through experiments in \Cref{sec:experiments}, showing that ReD solves more problems than STC at the same cost across different models and tasks. We summarize and discuss future work in \cref{sec:conclusions}.

\section{An illustrative example}
\label{sec:illustrative_ex}
\cref{fig:task} demonstrates what a verifiable task is in coding. The user asks an LLM to write a function to check whether any two elements in a list are within a certain threshold. The LLM completes the function incorrectly, with a bug that only checks every two adjacent elements on the list, and not all of them. The task is verifiable because there is an automatic way of checking if the answer is correct, through unit tests. In this case, the fourth test fails, since the code the LLM produced misses two close numbers that are not adjacent on the list. 

\begin{figure}[H]
\centering
\includegraphics[width=\linewidth]{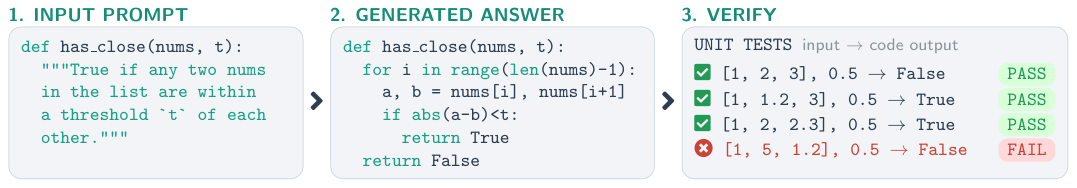}
\caption{\textbf{A verifiable coding problem.} Example of a coding problem, adapted from HumanEval, showing: 1) the input prompt, asking for a function that checks if \emph{any} two numbers in a list are within threshold 2) the model-generated answer and 3) the verifier, which runs unit tests on it. Each test runs the generated code on an input list and a threshold and outputs True or False. The fourth test fails, since the code checks only adjacent elements of the list and misses 1 and 1.2, which are within the threshold but not adjacent (the correct output is True). The coding problem is considered solved only if all tests pass, so this a wrong answer.}
\label{fig:task}
\end{figure}

\cref{fig:concept}a shows the performance of STC under a fixed budget. Consider a pool of coding problems similar to the one presented in \cref{fig:task}. This allocation policy goes over the problems one by one, moving to the next problem only when it is solved, i.e., when all the tests pass. As a result, when the policy reaches a very hard problem, it might exhaust all of its attempt budget on it and never reach the following problems. ReD is an allocation policy that is guaranteed to increase the coverage at any budget over STC and, as we will show later, over any multiple-independent-attempt allocation policy. As shown in \cref{fig:concept}b, ReD will allow only one attempt of solving a problem, and if it fails, it will put it back in the pool and move on to the next problem. Once all problems have been tried, the cycle is repeated. This allows ReD to not spend all of its budget on the hard problems, statistically sift out the easy problems first without knowing their difficulties in advance, and provably increase coverage@cost for any budget. 

\begin{figure}[H]
\centering
\includegraphics[width=0.85\linewidth]{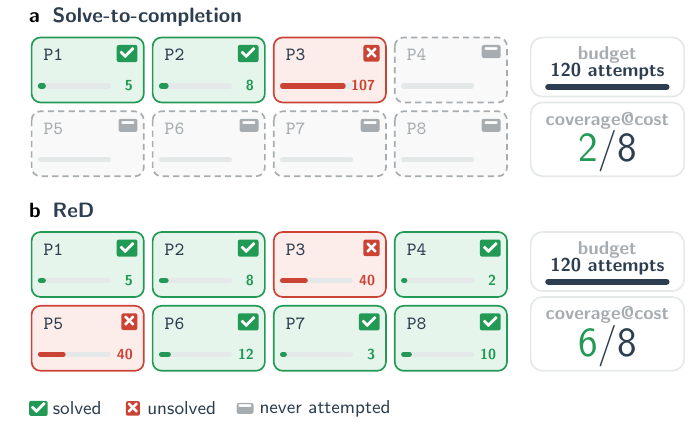}
\caption{\textbf{Solving problems under a fixed global budget.} Illustrative example: a pool of 8 coding problems similar to the one in \cref{fig:task}, with a budget of 120 attempts; bars and numbers show the attempts spent on each problem. \textbf{a} STC goes over the problems one by one, moving to the next problem only when it is solved. When it reaches a very hard problem (P3), it exhausts the rest of its attempt budget on it and never reaches the following problems. \textbf{b} ReD allows only one attempt at solving a problem. If it succeeds, it discards the problems and never tries solving it again. If it fails, it puts the problem back in the pool and moves to the next one. Once all problems have been tried, the cycle is repeated. This allows ReD to not spend all its budget on the hard problems (P3 and P5) and to statistically sift out the easy problems first, without knowing their difficulties in advance. ReD is optimal over all multiple-independent-attempt strategies at every budget (\cref{eq:OptimalityReD,eq:imperfect}).}
\label{fig:concept}
\end{figure}

\section{Mapping \texorpdfstring{pass@\textit{k}}{pass@k} to coverage@cost}
\label{sec:map}
\subsection{\texorpdfstring{pass@\textit{k}}{pass@k} and its asymptotic power-law behavior}
We follow similar notation to Levi~\cite{levi2024simple}, Schaeffer et al.~\cite{schaeffer2025large} and Kazdan et al.~\cite{kazdan2025efficient}. 
{For clarity, we first present the analysis under a perfect verifier. In \cref{app:noisy}, we relax this assumption and show that the ReD advantage extends to noisy verifiers with false-negative rates (FNR) and false-positive rates (FPR) satisfying FNR + FPR $< 1$.}
We denote the probability of success at a single attempt on the $i$-th problem as $p_i$ and the cumulative distribution function (CDF) of the number of attempts to answer it as $\text{pass}_i\text{@}k = 1-(1-p_i)^k$. Each problem has a different $p_i$ (smaller for difficult problems) drawn from some distribution $\mathcal{P}(p)$. Then, $\text{pass@}k$ is defined as the ensemble-averaged probability of at least one success in $k$ trials,
\begin{equation}
    \text{pass@}k \coloneqq \mathbb{E}\left[ \text{pass}_i\text{@}k \right] = 1 - \mathbb{E}\left[ (1-p_i)^k \right] = 1 - \int_0^1 (1-p)^k\mathcal{P}(p)\mathrm{d}p . 
\label{eq:pass@k}
\end{equation}
\Cref{eq:pass@k} is general and leads, under mild assumptions, to an asymptotic power-law behavior of pass@$k$. To see this, we consider difficulty distributions that for $p\ll 1$ have the form $\mathcal{P}(p)\simeq cp^{\alpha-1}$, where $c$ and $\alpha$ are positive constants. We then observe that $1 - \text{pass@}k=\frac{1}{k}\int_0^k (1-z/k)^k\mathcal{P}(z/k)\mathrm{d}z$, by a simple change of variable $z=pk$. Taking the large $k$ limit then yields 
\begin{equation}
    1 - \text{pass@}k \underset{k\gg 1}{\approx} 
    \frac{c}{k^\alpha}\int_0^\infty e^{-z}z^{\alpha-1}\mathrm{d}z =
     c \Gamma(\alpha) k^{-\alpha},
\label{eq:psi_scaling}
\end{equation}
where we have used the definition of the Gamma function, $\Gamma(\cdot)$.
From \cref{eq:psi_scaling}, it is clear that any distribution of difficulties, $\mathcal{P}(p)$, with power-law behavior at small $p$, results in a power-law scaling for $\text{pass@}k$, as also obtained by~Schaeffer et al.~\cite{schaeffer2025large}.

\subsection{Prediction of coverage@cost from \texorpdfstring{pass@\textit{k}}{pass@k}}

Consider a pool of problems, ordered randomly and an LLM which attempts to solve each problem repeatedly, until it is solved. The coverage@cost$(t)$ metric represents the mean number of unique problems answered under a given budget or cumulative number of attempts $t$. To connect that metric to pass@$k$, we define a single realization of that process:
we denote by $x(t)$ the number of unique problems answered by an LLM after a cumulative number of attempts $t$. Only when a problem is solved correctly, $x$ is increased by $1$ and the model attempts to solve the next problem in the pool until it is solved. In other words,
    \begin{equation}
    x(t) \coloneqq \sum_{n=1}^{\infty} \mathbb{I}_{\{J_n\leq t\}} = \text{sup}\left\{n: J_n \leq t \right\},    
\label{eq:x(t)}
\end{equation}
where, $\mathbb{I}$ is the indicator function and {$J_n \coloneqq \sum_{i=1}^{n}T^{(i)}$} is the cumulative number of attempts to answer problems $i=1,...,n$.
$T^{(i)}$ is the number of attempts to answer the $i$-th problem in the pool for the first time. Because all the problems are independent and identically distributed (i.i.d.), $T^{(i)}\sim T,\,\forall i$. We observe that the CDF of the random variable $T$ is $\text{pass@}t$ and denote,
\begin{equation}
    F(t) \coloneqq \text{Pr}(T\leq t)=\text{pass@}t.
\end{equation}
Taking the average over~\cref{eq:x(t)}, we get 
\begin{equation}
    \text{coverage@cost}(t) = \langle x(t)\rangle,
\end{equation}
    where we reserve the notation $\mathbb{E}[\cdot]$ for averages over the problem difficulty distribution and use brackets $\langle \cdot \rangle$ for averages over random realizations.
We recognize that \cref{eq:x(t)} is a renewal process and that $\text{coverage@cost}(t)$ can be predicted using the following recurrence relation~\cite{grimmett2020probability},
\begin{equation}
    \text{coverage@cost}(t) = F(t) + \sum_{j=1}^{t}\text{coverage@cost}(t-j) \left[ F(j) - F(j-1) \right].
\label{eq:renewal_cover@cost}
\end{equation}
See the full derivation of \cref{eq:renewal_cover@cost}, and the second moment of $x(t)$, in \cref{subsec:The_Renewal_Process}.
This is the first key result of the paper, connecting pass@k and coverage@cost through renewal theory, and enabling prediction of coverage@cost under STC, given pass@k.

\subsection{Growth of the coverage@cost for a large dataset}
Recall the definition of the Z-transform,
\begin{align*}
    \mathcal{Z}\{\text{coverage@cost}\}
    (z)=\sum_{t=0}^\infty z^t
    \text{coverage@cost}(t),
\end{align*}
where $z\in [0,1)$. We apply it to~\cref{eq:renewal_cover@cost}, and expand around $z=1$, to get the behavior of coverage@cost at times that are larger than the typical time to answer a problem in the dataset. As shown in \cref{subsec:Asymptotic_coverage@cost}, depending on the inference power-law exponent $\alpha$, two behaviors can be observed
\begin{equation}
    \begin{cases}
        \text{coverage@cost}(t) \propto t^{\alpha}&0<\alpha< 1,\\
        \text{coverage@cost}(t)\simeq\frac{t}{\mathbb{E}[T]}&\alpha>1,
    \end{cases}
\label{eq:asympt_coverage@cost}
\end{equation}
Namely, in the case where $0<\alpha<1$, $\mathbb{E}[T]$ diverges, and then coverage@cost will scale sub-linearly in $t$. However, if $\mathbb{E}[T]$ is finite, coverage@cost will scale linearly in $t$ (see~\Cref{fig:growth}).
This is the same transition that separates sub-diffusive from diffusive transport in continuous-time random walks~\citep{Scher1975Anomalous,metzler2000random}, where coverage@cost plays the role of the mean number of steps.

As we will show in the next section, ReD improves coverage@cost for all $t$.
Importantly, for the case where $\mathbb{E}[T]$ diverges, ReD will make it finite, making a dramatic, qualitative change in coverage@cost, moving from sublinear to linear growth. Analysis done in \cite{levi2024simple} on several models and datasets shows that the more common case is $\alpha<1$ which means $\mathbb{E}[T]$ diverges, underscoring the strength of the ReD methodology. If $\mathbb{E}[T]$ is finite, ReD will reduce the mean number of attempts to answer a randomly sampled problem, and therefore will increase the slope of the linear growth of coverage@cost as a function of $t$ (see~\Cref{fig:growth}). 

\section{ReD improves coverage@cost}
\label{sec:red_improves}

In ReD, we order the problems, and attempt to solve the problem at the front of the queue repeatedly up to $\tau$ times. Once it is solved, it is discarded from the queue. If it is still unsolved after $\tau$ attempts, we reset by moving it to the back of the queue, and continue to the next problem. We repeat this until the budget is spent or no problem is left unsolved. ReD therefore cycles through the problems in rounds, and in each round every unsolved problem receives up to $\tau$ attempts (see pseudo-code in \cref{app:algorithm}).

\subsection{Connecting \texorpdfstring{pass@\textit{k}}{pass@k} to coverage@cost with ReD}
\Cref{eq:renewal_cover@cost} establishes a connection between the pass@$k$ metric to the coverage@cost of STC. Here, we show how to obtain the $\text{coverage@cost}$ under ReD from pass@$k$. We first consider an infinitely large problem pool and address finite-size effects in \cref{sec:red_finite}. The ReD procedure is characterized by a different random variable $T_{\tau}$, representing the number of attempts to answer a single, randomly chosen, problem correctly under resetting, defined as,
\begin{equation}
    T_{\tau} = \begin{cases}T & \text { if } T \leqslant \tau, \\ \tau+T_{\tau}^{\prime} & \text { if } T>\tau,\end{cases}
\end{equation}
where, $T_{\tau}^{\prime}$ is an i.i.d. copy of $T_{\tau}$ and independent of $T$.

As a result, \Cref{eq:x(t)} is also valid for the ReD process, if we replace the random variable $T$ with $T_\tau$, and \cref{eq:renewal_cover@cost} is also applicable for the ReD protocol, if we replace $F(t)$ by $F_{\tau}(t)$, the CDF of $T_\tau$,
\begin{equation}
    \text{coverage@cost}_\tau(t) = F_\tau(t) + \sum_{j=1}^{t}\text{coverage@cost}_\tau(t-j) \left[ F_\tau(j) - F_\tau(j-1) \right].
\label{eq:renewal_cover@cost_tau}
\end{equation}
Next, we use a known connection between $F(t)=\text{pass@}t$ of the STC policy and $F_{\tau}(t)$,~\cite{eliazar2021tail}
\begin{equation}
    F_{\tau}(t) = 1 - \left(1-F(\tau)\right)^n \left(1-F(u)\right),
\label{eq:CDF_with_reset}
\end{equation}
where $n$ is the integer floor of $t/\tau$ and $u = t - n\tau$.
By plugging \cref{eq:CDF_with_reset} into \cref{eq:renewal_cover@cost_tau}, we can predict the $\text{coverage@cost}_{\tau}(t)$ with ReD every time interval $\tau$ from $F(k) =\text{pass@}k$ without resetting.

\subsection{Resetting always improves coverage@cost and resetting every attempt is the optimal strategy}
The mean number of attempts to answer a randomly-chosen problem correctly under resetting at interval $\tau$, $\mathbb{E}[T_{\tau}]$, can be found by evaluating the sum $\mathbb{E}[T_{\tau}]=\sum_{k=0}^{\infty} (1-F_{\tau}(k))$. Plugging \cref{eq:CDF_with_reset} we get \cite{eliazar2020mean,Lauber_Bonomo_2021,meir2025first},
\begin{equation}
    \mathbb{E}[T_{\tau}] = \frac{G(\tau)}{F(\tau)}, \quad G(\tau)=\sum_{k=0}^{\tau-1} (1-F(k)) .
\label{eq:MFPT}
\end{equation}
Importantly, we get that both $G(\tau)$ and $1/F(\tau)$ are finite for finite $\tau\geq1$, since $0<F(t)\leq 1$ for any $t\geq1$. It means that ${\mathbb{E}[T_{\tau}]}$ is also finite. As a result, only the bottom row of \cref{eq:asympt_coverage@cost} is relevant. Inserting \cref{eq:MFPT} results in linear growth of $\text{coverage@cost}_\tau(t)$,
\begin{equation}
    \text{coverage@cost}_{\tau}(t)\simeq\frac{t}{\mathbb{E}[T_{\tau}]} = \frac{t F(\tau)}{G(\tau)} . 
\label{eq:coverage@cost_always_lin}
\end{equation}
Comparing this equation with \cref{eq:asympt_coverage@cost}, one can immediately see that using ReD for the case where $\alpha<1$, will increase dramatically the coverage@cost (\cref{fig:growth}). Without ReD, one will get that coverage@cost will scale sub-linearly in time, and with resetting it will scale linearly, regardless the $\tau$ used, or the specific value of $\alpha<1$.

\begin{figure}[t]
\centering
\includegraphics[width=0.6\linewidth]{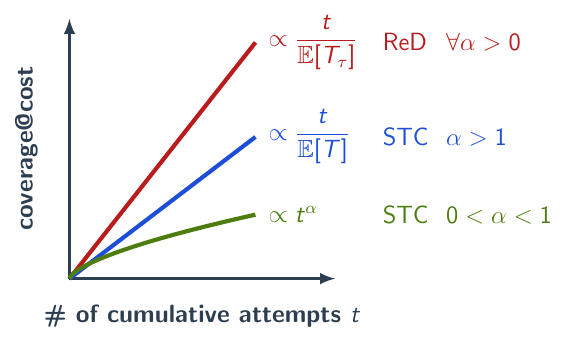}
\caption{\textbf{Growth of coverage@cost with the cumulative number of attempts $t$ (schematic).} STC (\cref{eq:asympt_coverage@cost}): The common case for LLMs is $0<\alpha<1$~\cite{levi2024simple}. In this case, $\mathbb{E}[T]$ diverges and coverage@cost scales sub-linearly, $\propto t^{\alpha}$. When $\alpha>1$, $\mathbb{E}[T]$ is finite and coverage@cost scales linearly, $\propto t/\mathbb{E}[T]$. ReD: $\mathbb{E}[T_\tau]$ is always finite (\cref{eq:MFPT}), so coverage@cost scales linearly, $\propto t/\mathbb{E}[T_\tau]$, regadless of $\alpha$ (\cref{eq:coverage@cost_always_lin}). For $\alpha>1$, since $\mathbb{E}[T]\geq\mathbb{E}[T_\tau]$ (\cref{eq:Resetting_is_always_beneficial}), ReD has the larger slope. ReD's slop is maximized for $\tau=1$ (\cref{eq:OptimalityReD}).}
\label{fig:growth}
\end{figure}

For the case where $\alpha>1$, where both STC and ReD exhibit linear growth, we show in \cref{eq:Resetting_is_always_beneficial} at \cref{sec:proof_ReD_opt}, that no matter what the underlying distribution of problem difficulties $\mathcal{P}(p)$ is, ReD is always beneficial. 
Finally, we prove in \cref{sec:proof_ReD_opt} that 
\begin{equation}
    \mathbb{E}[T_{\tau}]\leq \mathbb{E}[T_{\tau+1}]\,,\quad \text{for any } \mathcal{P}(p) \text{ and }  \tau = 1,2,3,... %
\label{eq:OptimalityReD}
\end{equation}
i.e., the optimal resetting time is $\tau=1$ regardless of the difficulty distribution $\mathcal{P}(p)$.

So far, we considered only \emph{deterministic} resetting intervals, but, in general, a \emph{stochastic} resetting time can be sampled from some distribution. However, it is known that a deterministic $\tau$ minimizes the mean time for completion of a task under resetting \cite{pal2017restart}. Therefore, \cref{eq:OptimalityReD} also shows that resetting deterministically after every attempt is the optimal resetting strategy.

We extend this result to imperfect verifiers in \cref{app:noisy}, where we prove that, for any non-degenerate $\mathcal{P}(p)$ and any verifier with $\mathrm{FNR}+\mathrm{FPR}<1$,
\begin{equation}
    \mathbb{E}[\mathcal{R}_{\text{ReD}}]=(1-\mathrm{FNR})\,\mathbb{E}[p]>\mathbb{E}[\mathcal{R}]\,,
\label{eq:imperfect}
\end{equation}
where $\mathbb{E}[\mathcal{R}]$ is the expected reward per attempt of any other multiple-independent-attempt strategy, i.e., the probability that a problem is both solved correctly and marked as correct by the verifier. Since $\mathrm{FNR}=\mathrm{FPR}=0$ satisfies $\mathrm{FNR}+\mathrm{FPR}<1$, \cref{eq:imperfect} shows that ReD is the better strategy at every budget, with or without verifier noise.

\subsection{Prediction of coverage@cost for a finite-sized dataset}
\label{sec:red_finite}
So far, to derive the equations for coverage@cost under ReD we neglected the effects of discarding the problems, assuming an infinite pool of problems, i.e., the probability of re-answering the same problem is zero. Here, we derive an expression for coverage@cost under ReD for a finite-sized dataset composed of $N$ problems and a resetting interval of $\tau=1$. In this scenario, we have repeated rounds where we attempt to answer each of the remaining problems once.

We first use pass@$k$ to find the number of remaining problems, $R_n$, after $n$ rounds. The mean number of unanswered problems at the end of the $n$-th round is, 
\begin{equation}
    \langle R_n\rangle = N(1-F(n))=N(1-\text{pass@}n).
    \label{eq:remainder}
\end{equation}
Because the probability that a problem survives after $n$ rounds is $1-F(n)$.
Then, 
\begin{equation}
    \text{coverage@cost}_{\tau=1}( t(n))=N-\langle R_n\rangle=N \text{pass@}n, 
    \label{eq:predict_coverage_finite}
\end{equation}
where $t(n)$ is the total number of attempts at the end of the $n$-th round.
In turn, the total number of attempts is given by,
\begin{equation}
    t(n) = \sum_{k=0}^{n-1}R_k.
    \label{eq:recursive_t}
\end{equation}
Taking averages and plugging \cref{eq:remainder} into \cref{eq:recursive_t}, gives
\begin{equation}
    \langle t(n)\rangle = N\sum_{k=0}^{n-1}(1-F(k)) = N\sum_{k=0}^{n-1}(1-\text{pass@}k).
    \label{eq:predict_time_coverage_finite_t}
\end{equation}
To predict coverage@cost$_{\tau=1}$ for a finite dataset, we replace $t(n)$ in \cref{eq:predict_coverage_finite} by its average $\langle t(n)\rangle$ to obtain the following approximation 
\begin{equation}
    \text{coverage@cost}_{\tau=1}(\langle t(n)\rangle)\approx N \text{pass@}n,
    \label{eq:predict_coverage_finite_t}
\end{equation}
for coverage@cost at fixed time points $t_n=\langle t(n)\rangle$. This approximation neglects fluctuations in $R_n$ and $t_n$. We show in~\cref{sec:experiments} that it works well for LLMs.  

\section{\texorpdfstring{Estimation}{Estimation} of \texorpdfstring{pass@\textit{k}}{pass@k} power-law exponent from coverage@cost with ReD}
\label{sec:inference_exponent}

In \cref{sec:red_improves}, we created a framework that predicts coverage@cost with ReD, given that pass@$k$ is already known. Here, we focus on the scenario where pass@$k$ is unknown and show that the ReD strategy enables estimation of the pass@$k$ scaling law exponent, without full evaluation of pass@$k$.

\subsection{ReD changes the difficulty distribution}

We first consider the distribution of problem difficulties $\mathcal{P}(p)$. In the ReD protocol, since we discard problems that were answered correctly in previous rounds, we effectively change the difficulty distribution at every round. We denote the distribution at the $n$-th round as $\mathcal{P}_n(p)$. We will now show that $\mathcal{P}_n(p)$ is guaranteed to approach a beta distribution after several rounds as long as, for $p\ll1$, the initial distribution $\mathcal{P}_0(p)\simeq cp^{\alpha-1}$, $\forall \alpha>0$. We then derive results for $\text{coverage@cost}_{\tau=1}(t)$ assuming $\mathcal{P}_0(p) \sim \text{Beta}(p;\alpha,\beta)$, neglecting the initial relaxation phase, i.e., the first rounds, before $\mathcal{P}_n(p)$ approaches this form. 

\subsection{Asymptotic convergence of difficulty to the Beta distribution} 
\label{subsec:Beta_is_attractor}
The distribution $\mathcal{P}_n(p)$ represents all the problems that were left unanswered after $n$ rounds. This is given by 
\begin{equation}
    \mathcal{P}_n(p) = \frac{\mathcal{P}_0(p)(1-p)^n}{\int_0^1 (1-p)^n\mathcal{P}_0(p)\mathrm{d}p},
\label{eq:updating_dist}
\end{equation}
where the new weight of problems with difficulty $p$ is given as a product of their initial weight $\mathcal{P}_0(p)$ and the probability, $(1-p)^n$, to survive $n$ rounds. 
Note that for large values of $n$ the term $(1-p)^n$ suppresses contributions from high values of $p$. In this limit, one can safely approximate $\mathcal{P}_0(p)\simeq cp^{\alpha-1}$, to obtain
\begin{equation}
    \mathcal{P}_n(p) \underset{n\gg 1}{\approx} \frac{p^{\alpha-1}(1-p)^n}{\int_0^1 p^{\alpha-1} (1-p)^n \mathrm{d}p}, 
\end{equation}
which is the probability density of $\text{Beta}(p;\alpha,n+1)$.

\subsection{Estimating the power-law exponent from ReD}
Previous works~\cite{brown2024llmonkeys, schaeffer2025large, kazdan2025efficient} have introduced several techniques to obtain the scaling exponent $\alpha$ of $1-\text{pass@}k$. Here, we derive an estimation method that is applicable to the ReD process.

Neglecting the relaxation towards the Beta distribution, i.e., assuming $\mathcal{P}_0(p)=\text{Beta}(p;\alpha,\beta)$, and plugging into \cref{eq:updating_dist}, we obtain,
\begin{equation}
    \mathcal{P}_n(p) = \text{Beta}(p;\alpha,\beta +n).
    \label{eq:P_n}
\end{equation}
The meaning of \cref{eq:P_n} is that $\gamma_n R_n$ problems will be answered correctly, on average, at the $n$-th round, where $\gamma_n \coloneqq \alpha/(\alpha+\beta+n)$ is the mean of $\text{Beta}(p;\alpha,\beta +n)$. As a result, the dynamics of this process can be written as,
\begin{equation}
    \langle R_{n+1}\rangle =\langle R_n\rangle - \gamma_n \langle R_n\rangle.
\label{eq:R_n_dynamics}
\end{equation}
By rearranging \cref{eq:R_n_dynamics}, and using the definition of $\gamma_n$, we obtain a linear equation with respect to the number of rounds,
\begin{equation}
    -\frac{\langle R_n\rangle}{\langle R_{n+1}-R_n\rangle} = \frac{1}{\alpha}n + \frac{\alpha+\beta}{\alpha} .
    \label{eq:inference_power_law}
\end{equation}
Thus, plotting the ratio $\langle R_n\rangle/\langle R_{n+1}-R_n\rangle$ versus $n$ yields a straight line whose slope is $1/\alpha$.
After a few rounds, it is possible to extract the scaling exponent $\alpha$ from a linear fit. 

\section{Experiments on large language models}
\label{sec:experiments}

To demonstrate the efficiency of ReD on LLMs, we ran experiments on the HumanEval dataset~\cite{chen2021evaluatinglargelanguagemodels} {and provide additional
validation on GSM8K~\cite{cobbe2021trainingverifierssolvemath} and MMLU-Pro~\cite{wang2024mmlu} in \cref{app:benchmarks}} and on SWE-bench Lite in \cref{app:swebench}.
We used two query protocols: The first is the STC method, in which we ask a problem repeatedly until it is answered correctly before moving to the next problem. The second is the ReD approach, resetting every attempt, which amounts to asking problems in rounds: we ask all the problems once and discard the ones that were answered correctly from the next round.

In practice, we performed these experiments by evaluating pass@$k$ up to $k=100$ once for each model and saving, for each problem and every attempt, whether it was answered correctly in a results matrix of dimension problems$\times$attempts. We also recorded how many input and output tokens were used. 
To generate many random realizations of the STC evaluation procedure, we shuffle rows and columns of the results matrix and analyze them as if we used the STC protocol. From the same shuffled data, we also generate all the realizations for the ReD protocol, by analyzing the results matrix in rounds, going column-by-column, and discarding rows that were answered in previous rounds. 

We performed the experiments using the Groq API (a commercial cloud service for running LLMs), and three different models: llama-3.1-8b-instant, llama-3.3-70b-versatile and openai/gpt-oss-20b~\cite{openai2025gptoss120bgptoss20bmodel}, at the default medium reasoning effort for the latter. All three models were given the same prompts. We set a strict system prompt (instructions given to the model before each problem) to reduce conversational filler and enforce a consistent output format: 
\textit{``You are an intelligent coding assistant. You will be provided with a function signature and docstring. You must complete the function body. Return ONLY the python code for the function implementation. Do not add explanations or tests. Wrap the code in a python markdown block."}
Then, we used the standard HumanEval prompt, including the function signature and docstring. Pattern matching (Regex) is used to extract the output code. We used a sampling temperature of 0.8 (a parameter that controls the randomness of the generated text) in all experiments.  

\begin{figure}%
    \centering
    \includegraphics[width=\linewidth]{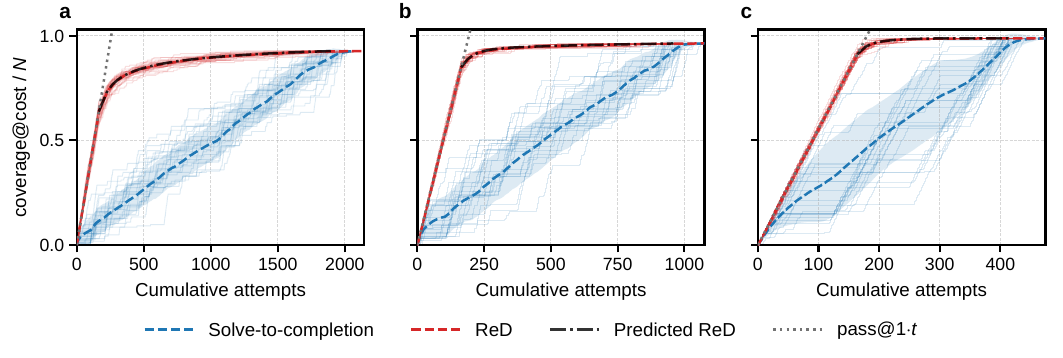}%
    \caption{coverage@cost normalized by the total number of problems, $N=164$, against cumulative number of attempts for ReD and the STC evaluation protocol for three models: (a) llama-3.1-8b-instant, (b) llama-3.3-70b-versatile, and (c) gpt-oss-20b. The dashed lines and shaded regions represent the mean and standard deviation over random realizations, respectively. The dashed-dotted line represents predictions using \cref{eq:predict_time_coverage_finite_t,eq:predict_coverage_finite_t}. The dotted line shows pass@$1\cdot t$, which is the expected behavior in the first ReD round.}
    \label{fig:LLM_exp1}
\end{figure}

\Cref{fig:LLM_exp1} shows the comparison between the coverage@cost, normalized by the total number of problems, as a function of the cumulative number of attempts for the three models using STC and ReD. For each model, we plot the mean coverage over 100 random realizations (dashed lines) and its standard deviation (shaded region). We also show 25 of the realizations used for each evaluation protocol. For all models, we see that the coverage@cost is always higher using ReD, on average, than the STC protocol, while leading to the same asymptotic coverage.

{\Cref{fig:LLM_exp2} compares the 8b model (STC and ReD) against the larger models (STC only) across three cost metrics, using Groq API pricing of \$0.05/\$0.08, \$0.59/\$0.79, and \$0.075/\$0.30 per 1M input/output tokens for the 8b, 70b, and gpt-oss-20b models, respectively. In terms of attempts, ReD on the 8b model is more efficient than STC with 70b model until ${\sim}90\%$ coverage and comparable to STC gpt-oss-20b until ${\sim}80\%$. In terms of tokens, gpt-oss-20b is considerably more verbose, so ReD on the 8b model outperforms it across nearly all coverage levels. In terms of USD, ReD on the 8b model is the most cost-effective strategy in all cases, since the larger models are either verbose or expensive per token. The preferred hybrid strategy is to run ReD with the small model first, then route only unsolved problems to the larger model.}

\begin{figure}%
    \centering
    \includegraphics[width=\linewidth]{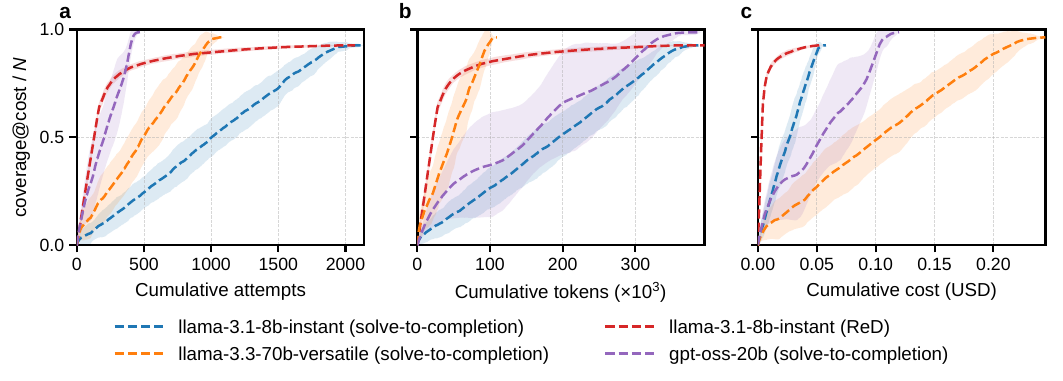}%
    \caption{coverage@cost normalized by the total number of problems, $N=164$, of llama-3.1-8b-instant (STC versus ReD) compared to llama-70b-versatile and gpt-oss-20b (STC only), measured by (a) Attempts, (b) Tokens, and (c) USD cost. The dashed lines and shaded regions represent the mean and standard deviation over random realizations, respectively.}
    \label{fig:LLM_exp2}
\end{figure}

Finally, we demonstrate in~\cref{fig:LLM_inference} the procedure to estimate the power-law exponent of the smaller llama model using ReD. We show that, using~\cref{eq:inference_power_law}, we are able to correctly estimate the power-law behavior from a linear fit of $\langle R_n\rangle/\langle R_{n+1} - R_n\rangle$ against $n$. The obtained power-law is $0.34\pm0.01$, which is in good agreement with the power-law estimated from the behavior of 1-pass@$k$ at high $k$ of $0.34$ (see \cref{fig:loglog_alpha}). Note that the linear regression is accurate as long as $R_n\gg1$, otherwise the inverse of $R_{n+1}-R_n$ introduces large noise. Therefore, the linear fit was done up to the 15th round. The estimation of the power-laws exponents for the two larger models is not shown because they solve almost all of the problems in very few rounds using ReD.

\textbf{Additional benchmarks and baselines.} We also tested ReD on GSM8K and MMLU-Pro with the two llama models (\cref{app:benchmarks}). For the smaller llama model on GSM8K, ReD reaches 81\% coverage after the first round, versus 34\% for STC with the same number of attempts. On SWE-bench Lite, where one attempt is a full run of a coding agent driven by DeepSeek-Coder-V2-Instruct~\cite{deepseek2024coderv2}, ReD reaches the same coverage as STC with 17--36 times fewer attempts and 18--38 times lower cost (\cref{app:swebench}).

Experiments with imperfect verifiers (\cref{app:noisy}) and a comparison with continuous reflection (re-prompting the model with its previous incorrect answer; \cref{app:baselines}) show that ReD also dominates with noisy verifiers and against other baselines. A hardware latency analysis that accounts for the reuse of the already-processed prompt across repeated attempts (KV-cache effects) is in \cref{app:kvcache}. Finally, we show in \cref{app:reasoning} that ReD’s saving in output tokens grows when the reasoning effort of gpt-oss-20b is raised from low to medium.

\begin{figure}%
    \centering
    \includegraphics[width=0.37\linewidth]{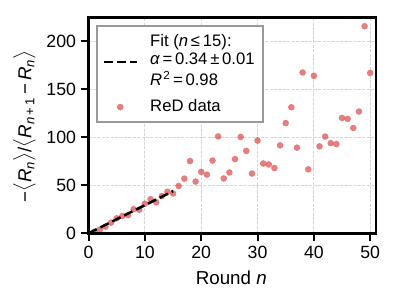}%
    \caption{The ratio $-\langle R_n\rangle/\langle R_{n+1} - R_n\rangle$ as a function of the round number $n$ for llama-3.1-8b-instant, where $R_n$ is the number of unsolved problems at round $n$.}
    \label{fig:LLM_inference}
\end{figure}

\section{Discussion}
\label{sec:conclusions}

We proposed Reset-and-Discard, ReD, a new resetting-based query protocol of LLMs that improves coverage@cost at a fixed budget.
We connected pass@$k$ to coverage@cost under STC and showed that power-law behavior in pass@$k$ with $0<\alpha<1$ leads to sublinear growth of coverage@cost. We proved that ReD always improves coverage@cost; particularly, in the case of $0<\alpha<1$, ReD leads to a linear growth of coverage@cost and the optimal resetting protocol is ReD every attempt. We validated these results across four LLMs and four benchmarks (coding, math, reasoning, GitHub issues), showing that ReD consistently dominates STC allocation across all budgets tested, and that the advantage is robust to realistic verifier noise.

This work addresses inference-time budget allocation for verifiable tasks, such as code generation, mathematical problem solving, and reasoning.
We further proved that the advantage of ReD persists for any verifier with $\text{FNR}+\text{FPR}<1$ (\cref{eq:imperfect}), which extends ReD to tasks that lack an exact verifier. This condition is met by a broad and growing set of methods that are not verifiers per se, such as another LLM that grades the answers (LLM-as-judge)~\cite{zheng2023judging}, a model trained to score them (a learned reward model)~\cite{lightman2024let} and partial unit tests. Building such methods for previously open-ended tasks is itself an active research direction~\cite{jia2025writingzerobridgegapnonverifiable}. Tasks with no notion of correctness, such as purely subjective generation, remain out of scope.
Our theoretical analysis abstracts inference as repeated independent attempts under a fixed prompt and decoding procedure; extending it to strategies that feed back previous failures, e.g., reflection, is a natural next step. In our preliminery experiments, ReD outperformed reflection (\cref{app:baselines}).

ReD suits services that charge for each solved problem rather than for each generated token (a shift from ``Pay-per-Token'' to ``Outcome-as-a-Service'')~\citep{chen2024frugalgpt}, because it solves more problems for the same budget.
The round in which ReD solves a problem also shows its difficulty: problems that need several rounds are those the model gets right only some of the time (the ``frontier of learnability'')~\citep{foster2026lilo}, which are the most useful for the model to be trained on~\citep{zelikman2022star,gulcehre2023rest}.
Since ReD discards solved problems, only the problems that a small model fails to solve need to be sent to a larger, more expensive model (hierarchical routing)~\citep{bansal2025smaller}.
In the same way, the problems that a model fails to solve at a low reasoning effort can be sent to a higher effort of the same model (\cref{app:reasoning}).
Finally, ReD only decides which problem receives the next attempt and does not change the model or the prompt, so it can be combined with methods that predict how hard each problem is or that further train the model.

\backmatter

\bmhead{Data availability}
The SWE-bench Lite samples analysed in \cref{app:swebench} were released by Brown et al.~\cite{brown2024llmonkeys} and are available at \url{https://github.com/ScalingIntelligence/swe-bench-lite-samples}.

\bmhead{Code availability}
All the code used to generate the results of this paper is available at \url{https://github.com/Hirshberg-Lab/ReD}.

\bibliography{refs}

\bmhead{Acknowledgements}
B.H. acknowledges support from the Israel Science Foundation (grants No. 1037/22 and 1312/22), the Pazy Foundation of the IAEC-UPBC (grant No. 415-2023). This project has received funding from the European Research Council (ERC) under the European Union’s Horizon 2020 research and innovation program (grant agreement No. 947731 to S.R.). S.M. acknowledges support from the Tel Aviv University Center for Artificial Intelligence and Data Science (TAD), and from the Israeli Planning and Budgeting Committee (VATAT). T.D.K. is supported by the Adams Fellowships Program of The Israel Academy of Sciences and Humanities.

\bmhead{Author contributions}
S.M. and T.D.K. derived most of the theoretical findings of this work. S.M. and B.H. performed all numerical experiments and analysis. S.R. and B.H. supervised the research. S.M., T.D.K., N.L., S.R., and B.H. wrote the manuscript.

\bmhead{Competing interests}
The authors declare no competing interests.

\begin{appendices}

\newpage

\renewcommand{\appendixname}{Supplementary Note}
\renewcommand{\thesection}{\arabic{section}}
\renewcommand{\figurename}{Supplementary Fig.}
\renewcommand{\tablename}{Supplementary Table}
\floatname{algorithm}{Supplementary Algorithm}
\renewcommand{\thefigure}{\arabic{figure}}
\renewcommand{\thetable}{\arabic{table}}
\renewcommand{\theequation}{S\arabic{equation}}
\setcounter{figure}{0} 
\setcounter{table}{0}
\setcounter{equation}{0}

\part*{\clearpage Supplementary Information}

\section{Additional benchmarks}
\label[appendix]{app:benchmarks}
We extended evaluation beyond HumanEval to GSM8K ($N=1{,}319$ math problems) and a subset of MMLU-Pro (first $N=500$, multi-discipline multiple-choice reasoning problems), on two model scales (llama-3.1-8b and llama-3.3-70b). We report below (\cref{tab:GSM8K} and \cref{tab:MMLU-Pro}) coverage@cost for these datasets, evaluated at $k\times N$ total attempts, averaged over 1,000 realizations.
\begin{table}[H]
    \centering
    \caption{GSM8K ($N=1{,}319$)}
    \begin{tabular}{llllll}
        Model & Method & $1 \times N$ & $2 \times N$ & $3 \times N$ & $5 \times N$ \\
        \hline 8b & STC & $33.6 \%$ & $66.7 \%$ & $98.3 \%$ & $99.0 \%$ \\
        \hline & ReD & $\mathbf{8 1 . 3} \%$ & $\mathbf{9 8 . 1 \%}$ & $99.0 \%$ & $99.0 \%$ \\
        \hline 70b & STC & $25.1 \%$ & $49.6 \%$ & $74.4 \%$ & $97.5 \%$ \\
        \hline & ReD & $\mathbf{9 3 . 2 \%}$ & $\mathbf{9 7 . 0 \%}$ & $\mathbf{9 7 . 3 \%}$ & $97.5 \%$ 
    \end{tabular}
    \label{tab:GSM8K}
\end{table}

\begin{table}[H]
    \centering
    \caption{MMLU-Pro ($N=500$)}
    \begin{tabular}{lllll}
        Model & Method & $1 \times N$ & $3 \times N$ & $5 \times N$ \\
        \hline 8b & STC & $10.4 \%$ & $30.3 \%$ & $50.3 \%$ \\
        \hline & ReD & $\mathbf{4 6 . 5} \%$ & $\mathbf{7 6 . 1} \%$ & $\mathbf{8 7 . 0} \%$ \\
        \hline 70b & STC & $9.7 \%$ & $28.2 \%$ & $46.5 \%$ \\
        \hline & ReD & $\mathbf{7 1 . 8} \%$ & $\mathbf{8 7 . 1} \%$ & $\mathbf{9 0 . 3} \%$
    \end{tabular}
    \label{tab:MMLU-Pro}
\end{table}
At a range of total attempts, ReD substantially outperforms STC for both datasets and both models, demonstrating its utility across three task types (coding, math, reasoning).

From a theoretical standpoint, as long as there is some variability among problems in the dataset, we proved that ReD must accelerate the inference.

\section{SWE-bench Lite: a coding agent on real GitHub issues}
\label[appendix]{app:swebench}
To test ReD on a harder task, we analysed the samples that Brown et al.~\cite{brown2024llmonkeys} released for SWE-bench Lite~\cite{jimenez2024swebenchlanguagemodelsresolve}. This benchmark consists of 300 real GitHub issues from Python repositories. An answer is a change to the code, and it is correct if the unit tests of the repository pass. The model, DeepSeek-Coder-V2-Instruct~\cite{deepseek2024coderv2}, worked inside a coding agent (Moatless Tools, \url{https://github.com/aorwall/moatless-tools}), which was run independently 250 times on each issue at a sampling temperature of 1.6. Here, one attempt is a full run of the agent, in which the model takes 9.9 steps on average (such as searching the code, planning and editing), rather than a single answer. We did not run any new inference.

For 34 of the 300 issues, the unit tests do not always give the same result for the same answer, so Brown et al. ran them 11 times and took the majority result. We used the other 266 issues, so that all answers are checked in the same way. On these issues, pass@1 is 14.4\%. As for HumanEval (\cref{sec:experiments}), we simulated ReD and STC from the same results matrix by shuffling, with 1,000 realizations.

\begin{table}[H]
    \centering
    \caption{Budget needed to reach a target coverage on SWE-bench Lite (266 issues, 1,000 realizations). Saving: the budget of STC divided by that of ReD.}
    \begin{tabular}{lrrccc}
        Coverage$/N$ & \multicolumn{2}{c}{Attempts} & \multicolumn{3}{c}{Saving} \\
         & STC & ReD & Attempts & Output tokens & USD \\
        \hline $5\%$ & 3,340 & 95 & 35.2 & 36.9 & 37.8 \\
        \hline $10\%$ & 6,448 & 182 & 35.5 & 37.3 & 38.0 \\
        \hline $15\%$ & 9,576 & 293 & 32.7 & 34.2 & 34.7 \\
        \hline $20\%$ & 13,044 & 536 & 24.3 & 25.2 & 25.5 \\
        \hline $25\%$ & 16,165 & 922 & 17.5 & 18.0 & 18.1
    \end{tabular}
    \label{tab:swebench}
\end{table}

ReD reaches every target with 17--36 times fewer attempts, 18--37 times fewer output tokens and 18--38 times lower cost (\cref{tab:swebench}). We note that since Brown et al.~\cite{brown2024llmonkeys} used only up to 250 samples per issue, these estimates are lower bounds on the savings. The USD costs are those recorded by the agent, at \$3 and \$15 per million input and output tokens.

On SWE-bench Lite, the estimated power-law exponent of $1-\text{pass@}k$ is $\alpha=0.12$, about a third of that of llama-3.1-8b-instant on HumanEval ($\alpha=0.34$, \cref{fig:loglog_alpha}). By \cref{eq:asympt_coverage@cost}, a smaller $\alpha$ makes coverage@cost grow more slowly under STC, whereas under ReD it grows linearly for any $\alpha$ (\cref{eq:coverage@cost_always_lin}). The large savings on SWE-bench Lite agree with this prediction.

\section{Pseudo-code of ReD}
\label[appendix]{app:algorithm}
\Cref{alg:ReD} gives the pseudo-code of ReD (\cref{sec:red_improves}).

\begin{algorithm}[H]
\footnotesize
\caption{Reset-and-Discard (ReD)}
\label{alg:ReD}
\begin{algorithmic}[1]
\State \textbf{Input:} Queue $L=[q_1,\dots,q_N]$ of problems to try, per problem attempt limit $\tau$, attempt budget $B$
\State $\texttt{Coverage} \gets \emptyset$; $t \gets 0$
\Comment{$\texttt{Coverage}$ stores solved problems; $t$ counts attempts}

\While{$L$ \textbf{is not} empty \textbf{and} $t < B$}
    \State $\texttt{Solved} \gets \texttt{False}$ ; $r \gets 0$; $q \gets \mathrm{Front}(L)$
    \Comment{look at the next problem to try}

    \While{$r < \tau$ \textbf{and} $t < B$ \textbf{and not} $\texttt{Solved}$}
        \State $r \gets r+1$; $t \gets t+1$

        \If{Attempt $q$ is successful}
            \State $\texttt{Coverage} \gets \texttt{Coverage} \cup \{q\}$; $\texttt{Solved} \gets \texttt{True}$; $\mathrm{Dequeue}(L)$
            \Comment{$q$ is solved and discarded from $L$}
        \EndIf
    \EndWhile

    \If{\textbf{not} $\texttt{Solved}$}
        \State  $\mathrm{Dequeue}(L)$; $\mathrm{Enqueue}(L,q)$
        \Comment{reset: move $q$ from the front to the back of $L$}
    \EndIf
\EndWhile

\State \textbf{Return} $\texttt{Coverage}$
\end{algorithmic}
\end{algorithm}

\section{Connecting coverage@cost to pass@k for a large dataset}
In this Supplementary Note, we recap known results in renewal theory \cite{grimmett2020probability}. We reproduce the proofs below for completeness and to facilitate reading.
\subsection{\texorpdfstring{Proof of \cref{eq:renewal_cover@cost}}{Proof of equation (renewal)}}
\label[appendix]{subsec:The_Renewal_Process}

In the main text, we defined $x(t)$ as follows,
\begin{equation}
    x(t) \coloneqq \sum_{n=1}^{\infty} \mathbb{I}_{\{J_n\leq t\}} = \text{sup}\left\{n: J_n \leq t \right\}, \quad J_n\coloneqq\sum_{i=1}^{n}T^{(i)},
\end{equation}
where $T^{(i)}$ are i.i.d. positive random variables distributed as $T$. Given that the first problem was answered correctly at step $j$, renewal processes have the following recursion relation,
\begin{equation}
    x(t)\stackrel{d}{=}
    \begin{cases}
        1 + x(t-j) &\text{if } j\leq t, \\
        0&\text{if } j>t.
    \end{cases}
    \label{eq:si_recursion}
\end{equation}
In other words, for $t<j$, we have $x(t)=0$, otherwise, $x(t)\stackrel{d}{=} 1 + x(t-j)$. This is because the first success was at time $j$, after which, the process renews itself and the number of additional successes from this time onward is distributed as $x(t-j)$.

To prove \cref{eq:renewal_cover@cost}, we introduce the following notations: $m_1(t)\coloneqq\text{coverage@cost}(t)$ and $f(j)\coloneqq F(j)-F(j-1)$, where $F(j)$ is the CDF of $T$ and $f(j)$ is its probability mass function. Note that with this notation, \cref{eq:renewal_cover@cost}, has the following form
\begin{equation}
   m_1(t) = F(t) + \sum_{j=1}^{t}m_1(t-j) f(j).
   \label{eq:eq6_SI}
\end{equation}
Next, we use the recursion relation in \cref{eq:si_recursion} with the law of total expectation, $\langle x\rangle=\langle\langle x\mid y\rangle\rangle$, where $\langle x\mid y\rangle$ is the conditional expectation of $x$ given $y$.  
\begin{equation}
\begin{split}
    &m_1(t)=\langle x(t)\rangle = \langle\langle x(t)\mid T^{(1)} \rangle\rangle = \sum_{j=1}^{\infty} \langle[x(t)\mid T^{(1)}=j]\rangle f(j) = \\
    &\sum_{j=1}^{\infty} \mathbb{I}_{\{j \leq t \}}(1+\langle x(t-j)\rangle) f(j) \\
    &= \sum_{j=1}^t (1+m_1(t-j)) f(j) = F(t) + \sum_{j=1}^{t}m_1(t-j)f(j).
\end{split}
\end{equation}
Replacing $m_1(t)=\text{coverage@cost}(t)$ gives \cref{eq:renewal_cover@cost}.
Similarly, we now obtain the second moment.

We use the recursive relation in \cref{eq:si_recursion}, and plug $x^2(t)\stackrel{d}{=}(1 + x(t-j))^2 = 1+2x(t-j) + x^2(t-j)$ for time $j\leq t$.
\begin{equation}
\begin{split}
    &m_2(t)=\langle x^2(t)\rangle = \langle\langle x^2(t)\mid T^{(1)}\rangle\rangle = \sum_{j=1}^{\infty} \langle[x^2(t)\mid T^{(1)}=j]\rangle f(j) = \\
    &\sum_{j=1}^{\infty} \mathbb{I}_{\{j \leq t \}}\left(1+2\langle x(t-j)\rangle+\langle x^2(t-j)\rangle \right) f(j) =\\
    &\sum_{j=1}^{t} (1+2m_1(t-j)+m_2(t-j)) f(j) =\\
    &\sum_{j=1}^{t} (2+2m_1(t-j)-1+m_2(t-j)) f(j) = 2m_1(t) - F(t) + \sum_{j=1}^{t}m_2(t-j)f(j). 
\end{split}
\end{equation}
Overall, we get
\begin{equation}
    m_2(t) = 2m_1(t) - F(t) + \sum_{j=1}^{t}m_2(t-j)f(j).
\label{eq:renewal_2nd_moment}
\end{equation}

\subsection{\texorpdfstring{Proof of \cref{eq:asympt_coverage@cost}}{Proof of equation (asymptotic coverage@cost)}}
\label[appendix]{subsec:Asymptotic_coverage@cost}
We start by taking the $Z$-transform of \cref{eq:eq6_SI,eq:renewal_2nd_moment}, to get
\begin{equation}\label{eq: renewal z trans}
\begin{split}
    \mathcal{Z}\{m_1\}(z)&=\frac{\Tilde{f}(z)}{(1-z)\left(1-\Tilde{f}(z)\right)},\\
    \mathcal{Z}\{m_2\}(z)&=\frac{\Tilde{f}(z)\left(1+\Tilde{f}(z)\right)}{(1-z)\left(1-\Tilde{f}(z)\right)^2}.
\end{split}
\end{equation}
Where $\Tilde{f}(z)\equiv\sum_{t=0}^\infty f(t)z^t$ is the $Z$-transform of the probability mass function of the time of answering a problem. We used the relation between the $Z$-transform of the CDF and the probability mass function, $\Tilde{F}(z)=\Tilde{f}(z)/(1-z)$, and the convolution theorem for $Z$-transforms \cite{feller1991introduction}.

Taking the long-time limit is equivalent to taking the limit $z\to1$, and inverting the transform. All inversions were done using the Tauberian theorem for $Z$-transforms \cite{feller1991introduction}. If $1-\text{pass@}k$ decays to zero faster than $ck^{-1}$, with $c$ being some constant, the mean time to answer a problem will be $\mathbb{E}[T]$, and $\Tilde{f}(z)=1-(1-z)\mathbb{E}[T]+o(1-z)$. Therefore, the asymptotic behavior around $z=1$ of \cref{eq: renewal z trans} is
\begin{equation}
    \begin{split}
    \mathcal{Z}\{m_1\}(z)&\simeq\frac{1}{(1-z)^2\mathbb{E}[T]}\Rightarrow m_1(t)\simeq\frac{t}{\mathbb{E}[T]},\\
    \mathcal{Z}\{m_2\}(z)&\simeq\frac{2}{(1-z)^3\mathbb{E}^2[T]}\Rightarrow m_2(t)\simeq\frac{t^2}{\mathbb{E}^2[T]}.
\end{split}
\end{equation}
We got that in this case, coverage@cost scales linearly in time, and the slope is the inverse of the mean number of attempts to answer a single problem.

It is also evident that the standard deviation of the number of unique problems answered $\sqrt{m_2(t)-m_1^2(t)}$ grows slower than $t$ at long times. Therefore, the ratio of it with coverage@cost approaches $0$ at late times. This means that in this scenario $x(t)/m_1(t)$ becomes deterministic and coverage@cost provides a good description of the dynamics. 

The other case is the one in which $1-\text{pass@}k$ decays as $ck^{-\alpha}/\Gamma(1-\alpha)$, with $0<\alpha<1$, for which we get that around $z=1$, $\Tilde{f}(z)= 1-c(1-z)^\alpha+o((1-z)^\alpha)$. Then
\begin{equation}
\begin{split}
    \mathcal{Z}\{m_1\}(z)&\simeq\frac{1}{c(1-z)^{1+\alpha}}\Rightarrow m_1(t)\simeq\frac{t^\alpha}{c\Gamma(1+\alpha)},\\
    \mathcal{Z}\{m_2\}(z)&\simeq\frac{2}{c^2(1-z)^{1+2\alpha}}\Rightarrow m_2(t)\simeq\frac{2t^{2\alpha}}{c^2\Gamma(2\alpha+1)}.
\end{split}
\end{equation}
Observe that here, the behavior of coverage@cost becomes sub-linear, with a power controlled by the inference scaling power $\alpha$. 

Regarding $\sigma_{x(t)}$ the standard deviation in $x(t)$, and its ratio with coverage@cost, here they approach
\begin{equation}
    \begin{split}
        \sigma_{x(t)}\simeq\sqrt{\frac{2}{\Gamma(2\alpha+1)}-\frac{1}{\Gamma^2(\alpha+1)}}\frac{t^\alpha}{c},\\
        \frac{\sigma_{x(t)}}{m_1(t)}\simeq\sqrt{\frac{2\Gamma^2(\alpha+1)}{\Gamma(2\alpha+1)}-1}\in(0,1).
    \end{split}
\end{equation}
Here, the ratio between the standard deviation and the mean number of unique problems answered approaches a constant. The value of this constant varies with the value of the inference scaling power, and it is monotonically decreasing for $0<\alpha<1$.

The transition between the linear and sub-linear behavior of coverage@cost appears when $\alpha=1$. For $\alpha>1$, the mean time to answer a randomly selected problem is finite, and for $\alpha<1$ it diverges. Mathematically, this is the same behavior as the well-studied transition between diffusive and sub-diffusive behavior in continuous-time random walk models used in the physics of transport phenomena \cite{Scher1975Anomalous, metzler2000random}.

\section{\texorpdfstring{Proof of \cref{eq:OptimalityReD}}{Proof of the optimality of ReD theorem}}
\label[appendix]{sec:proof_ReD_opt}

\begin{proof}
    We will show that $\mathbb{E}[T_{\tau+1}]\geq\mathbb{E}[T_\tau]$ for $\tau\geq1$ and for any $\mathcal{P}(p)$. We do so by defining the following difference $D(\tau) \coloneqq \mathbb{E}[T_{\tau+1}]-\mathbb{E}[T_{\tau}]$ and demonstrate that it is non-negative. We write the difference using \cref{eq:MFPT} %
    \begin{equation}
    \begin{split}
         D(\tau) = \frac{G(\tau+1)}{F(\tau+1)} - \frac{G(\tau)}{F(\tau)} = 
         \frac{F(\tau)G(\tau+1)-F(\tau+1)G(\tau)}{F(\tau)F(\tau+1)} .
    \end{split}\label{eq:D(tau)_new}
    \end{equation}
    Because the denominator in \cref{eq:D(tau)_new} is non-negative, it is left to show that the numerator is non-negative as well. 
    We start by recalling that 
    \begin{equation}
           F(\tau) \coloneqq \text{Pr}(T\leq \tau) = 1 - \int_0^1 (1-p)^{\tau}\mathcal{P}(p)\mathrm{d}p =1-\mathbb{E}\left[(1-p)^{\tau}\right], 
    \end{equation}
    and
    \begin{equation}
    \begin{split}
        &G(\tau)\coloneqq\sum_{k=0}^{\tau-1} (1-F(k))= \int_0^1 \sum_{k=0}^{\tau-1} (1-p)^k\mathcal{P}(p)\mathrm{d}p = 
        \int_0^1 \frac{1-(1-p)^{\tau}}{p}\mathcal{P}(p)\mathrm{d}p = \\
        &\mathbb{E}\left[\frac{1-(1-p)^{\tau}}{p}\right].
    \end{split}
    \end{equation}
    We can thus write the numerator of \cref{eq:D(tau)_new} as 
    \begin{equation}\label{eq:numerator_of_D(tau)_new}
           F(\tau)G(\tau+1)-F(\tau+1)G(\tau)= \mathbb{E}\left[1-q^{\tau}\right]\mathbb{E}\left[\frac{1-q^{\tau+1}}{p}\right]-\mathbb{E}\left[1-q^{\tau+1}\right]\mathbb{E}\left[\frac{1-q^{\tau}}{p}\right] .
        \end{equation}
    where we have set $q\coloneqq 1-p$. We now observe that $\mathbb{E}\left[1-q^{\tau+1}\right]=\mathbb{E}\left[1-(1-p)q^{\tau}\right]=\mathbb{E}\left[1-q^{\tau}\right]+\mathbb{E}\left[pq^{\tau}\right]$, and similarly $\mathbb{E}\left[\frac{1-q^{\tau+1}}{p}\right]=\mathbb{E}\left[\frac{1-(1-p)q^{\tau}}{p}\right]=\mathbb{E}\left[\frac{1-q^{\tau}}{p}\right]+\mathbb{E}\left[q^{\tau}\right]$. Substituting back into the right hand side of  \cref{eq:numerator_of_D(tau)_new} we obtain
    \begin{equation}
            F(\tau)G(\tau+1)-F(\tau+1)G(\tau)=\mathbb{E}\left[1-q^{\tau}\right]\mathbb{E}\left[q^\tau\right]-\mathbb{E}\left[\frac{1-q^{\tau}}{p}\right]\mathbb{E}\left[pq^\tau\right].
        \label{eq:numerator_of_D(tau)_new2}
        \end{equation}    
We now define $u\coloneqq (1-q^{\tau})/p$ and continue to develop the right hand side \cref{eq:numerator_of_D(tau)_new2}
    \begin{equation}\label{eq:numerator_par_new1}
    \begin{split}
        &\mathbb{E}\left[ 1 - q^{\tau} \right]\mathbb{E}\left[ q^{\tau} \right] - \mathbb{E}\left[ pq^{\tau} \right]  \mathbb{E}[u] = 
        \mathbb{E}[p u]\mathbb{E}[q^{\tau}] - \mathbb{E}[p q^{\tau}]  \mathbb{E}[u] = \\
        &\mathbb{E}[u]\mathbb{E}[q^{\tau}] \left(\frac{\mathbb{E}[p u]}{\mathbb{E}[u]}  - \frac{\mathbb{E}[p q^{\tau}]}{\mathbb{E}[q^{\tau}]}\right) = \mathbb{E}[u]\mathbb{E}[q^{\tau}] \left(\frac{\mathbb{E}[p u]}{\mathbb{E}[u]}  - \frac{\mathbb{E}[u p q^{\tau}/u] \mathbb{E}[u]}{\mathbb{E}[u]\mathbb{E}[uq^{\tau}/u]}\right) .
    \end{split}
    \end{equation}
    To finish the proof, we define $\mathbb{E}_{\lambda}[p] = \mathbb{E}[\lambda(p) p]/\mathbb{E}[\lambda(p)] $ as the expectation with respect to the probability density $\mathcal{P}(p)\lambda(p)/\mathbb{E}[\lambda(p)]$ for a positive function $\lambda(p)$.
    The parenthesis on the right hand side of \cref{eq:numerator_par_new1} can then be written as 
    \begin{equation*}
    \begin{split}
        \mathbb{E}_{u}[p] - \mathbb{E}_{q^{\tau}}[p] = \frac{\mathbb{E}_{u}[p]\mathbb{E}_{u}[q^{\tau}/u] - \mathbb{E}_{u}[p q^{\tau}/u] }{\mathbb{E}_{u}[q^{\tau}/u]} =
        - \frac{\text{COV}_u(p,q^{\tau}/u)}{\mathbb{E}_{u}[q^{\tau}/u]} \geq 0.
    \end{split}
    \end{equation*}
    The last inequality is because $q^\tau/u\geq0$ and because $q^{\tau}/u=p(1-p)^{\tau}/(1-(1-p)^{\tau})$ is strictly \emph{decreasing} with $p$, leading to a negative covariance according to the continuous version of the Chebyshev sum inequality. Thus, $D(\tau) \coloneqq \mathbb{E}[T_{\tau+1}]-\mathbb{E}[T_{\tau}] \geq 0$, which concludes the proof.   
    
\end{proof}

As a result, for any $\mathcal{P}(p)$, and $\forall\tau\in \mathbb{N}\setminus\{0\}$,
\begin{equation}
    \mathbb{E}[T]\geq \mathbb{E}[T_{\tau}]
\label{eq:Resetting_is_always_beneficial}
\end{equation}
\begin{proof}
    Recall $D(\tau)\geq 0$ and note that taking the limit $\tau\to\infty$ in \cref{eq:MFPT} gives $\lim\limits_{\tau\to\infty}\mathbb{E}[T_{\tau}]=\mathbb{E}[T]$. Therefore, $\mathbb{E}[T] = \sup\{\mathbb{E}[T_{\tau}] \}$.
    
\end{proof}

\section{Imperfect verifier}
\label[appendix]{app:noisy}
Here we address the case of an imperfect verifier and demonstrate experimentally and theoretically that ReD leads at a range of budgets.
\subsection{Experiments with imperfect verifiers}
We evaluated actual coverage under three noise conditions (FPR = false positive rate, FNR = false negative rate). The results are given for the GSM8K dataset (\cref{tab:GSM8K_Imperfect}) and MMLU-Pro (\cref{tab:MMLU-Pro_Imperfect}).
\begin{table}[H]
    \centering
    \caption{Coverage with an imperfect verifier across a range of noise conditions for the GSM8K ($N=1{,}319$, llama-3.1-8b, 1,000 realizations) dataset.}
    \begin{tabular}{lllll}
        Verifier & Method & $1 \times N$ & $2 \times N$ & $3 \times N$ \\
        \hline Perfect & STC & 33.6\% & 66.7\% & 98.3\% \\
        \hline & ReD & \textbf{81.3}\% & \textbf{98.1}\% & 99.0\% \\
        \hline $\mathrm{FPR}=0.02, \mathrm{FNR}=0.1$ & STC & 41.5\% & 83.1\% & 97.3\% \\
        \hline & ReD & \textbf{73.2}\% & \textbf{96.9}\% & 97.3\% \\
        \hline $\mathrm{FPR}=0.08, \mathrm{FNR}=0.15$ & STC & 52.1\% & 94.4\% & 94.4\% \\
        \hline & ReD & \textbf{69.2}\% & 94.4\% & 94.4\% \\
        \hline $\mathrm{FPR}=0.25, \mathrm{FNR}=0.05$ & STC & 67.5\% & 90.3\% & 90.3\% \\
        \hline & ReD & \textbf{77.3}\% & 90.3\% & 90.3\% 
    \end{tabular}
    \label{tab:GSM8K_Imperfect}
\end{table}

\begin{table}[H]
    \centering
    \caption{Coverage with an imperfect verifier across a range of noise conditions for the MMLU-Pro ($N=500$, llama-3.1-8b, 1,000 realizations) dataset.}
    \begin{tabular}{lllll}
        Verifier & Method & $1 \times N$ & $3 \times N$ & $5 \times N$ \\
        \hline Perfect & STC & 10.4\% & 30.3\% & 50.3\% \\
        \hline & ReD & \textbf{46.5}\% & \textbf{76.1}\% & \textbf{87.0}\% \\
        \hline $\mathrm{FPR}=0.02, \mathrm{FNR}=0.1$ & STC & 12.9\% & 38.4\% & 63.7\% \\
        \hline & ReD & \textbf{41.8}\% & \textbf{73.1}\% & \textbf{83.5}\% \\
        \hline $\mathrm{FPR}=0.08, \mathrm{FNR}=0.15$ & STC & 18.8\% & 56.1\% & 75.0\% \\
        \hline & ReD & \textbf{39.5}\% & \textbf{70.4}\% & 75.0\% \\
        \hline $\mathrm{FPR}=0.25, \mathrm{FNR}=0.05$ & STC & 29.5\% & 63.1\% & 63.1\% \\
        \hline & ReD & \textbf{44.2}\% & 63.1\% & 63.1\% 
    \end{tabular}
    \label{tab:MMLU-Pro_Imperfect}
\end{table}
All numbers in \cref{tab:GSM8K_Imperfect} and \cref{tab:MMLU-Pro_Imperfect} are actual coverage (a problem counts only when a truly correct answer is accepted). At an intermediate number of attempts, a non-zero FPR raises the coverage of STC (false positives allow it to stop trying to solve hard problems sooner), while reducing ReD's, narrowing the gap. A non-zero FPR also lowers asymptotic actual coverages for both approaches. Despite that, ReD leads at a range of budgets across all noise conditions. Under high noise (FPR=0.25 or FNR=0.15), the gap closes faster on the easier dataset (GSM8K), while on the harder dataset (MMLU-Pro) the advantage persists to substantially larger budgets.

In addition, we ran the imperfect verifier simulation on HumanEval (\cref{tab:HumanEval}) across a range of realistic noise conditions.
\begin{table}[H]
    \centering
    \caption{Coverage with an imperfect verifier across a range of realistic noise conditions for the HumanEval ($N=164$, llama-3.1-8b, 200 realizations) dataset}
    \begin{tabular}{llllll}
        FPR & FNR & Method & $1 \times N$ & $3 \times N$ & $5 \times N$ \\
        \hline 0\% & 0\% & STC & 9.3\% & 24.5\% & 39.8\% \\
        \hline & & ReD & \textbf{63.7}\% & \textbf{84.7}\% & \textbf{88.4}\% \\
        \hline 2\% & 2\% & STC & 15.5\% & 42.2\% & 67.7\% \\
        \hline & & ReD & \textbf{62.3}\% & \textbf{83.5}\% & \textbf{86.8}\% \\
        \hline 2\% & 10\% & STC & 14.8\% & 40.7\% & 66.1\% \\
        \hline & & ReD & \textbf{57.3}\% & \textbf{83.1}\% & \textbf{86.5}\% \\
        \hline 5\% & 5\% & STC & 21.9\% & 62.7\% & 83.9\% \\
        \hline & & ReD & \textbf{60.6}\% & \textbf{82.4}\% & 83.8\% \\
        \hline 8\% & 8\% & STC & 25.5\% & 75.1\% & 81.3\% \\
        \hline & & ReD & \textbf{58.9}\% & \textbf{81.1}\% & 81.5\% \\
        \hline 8\% & 15\% & STC & 24.9\% & 72.4\% & 81.1\% \\
        \hline & & ReD & \textbf{54.3}\% & \textbf{80.5}\% & 81.0\% 
    \end{tabular}
    \label{tab:HumanEval}
\end{table}
ReD leads at a range of budgets across the entire realistic noise range.

\subsection{\texorpdfstring{Proof of~\cref{eq:imperfect}}{Proof of the imperfect-verifier theorem}}
The point-wise advantage of ReD holds for the actual coverage, regardless of the verifier's FPR and FNR.
Consider the case of an infinite pool of problems whose distribution of probabilities to answer a problem correctly has a non-degenerate density  $\mathcal{P}(p)$; and a verifier with some false-negative rate, $\text{FNR}$, and some false-positive rate, $\text{FPR}$, such that $\text{FNR}+\text{FPR}<1$. We will compute the mean reward $\mathbb{E}[\mathcal{R}]$ per attempt, i.e., the probability that a problem is both solved correctly and marked as correct by the verifier.

In ReD, on an infinite problem pool, each problem is given a single attempt. The reward is therefore
\begin{equation}
    \mathbb{E}[\mathcal{R}_{\text{ReD}}]=\int_0^1 p(1-\text{FNR})\mathcal{P}(p)\,dp=(1-\text{FNR})\mathbb{E}[p].
\end{equation}
For a general (non-ReD) strategy, we will first compute the mean reward on a given attempt, given that the current problem being asked was asked $n-1$ times before
\begin{equation}
\begin{split}
    &\pi_n=\mathbb{E}[\mathcal{R}|\text{$n-1$ previous attempts}]= \\
    &\frac{\int_0^1 \left(1-\text{FNR}\right)p\left[1-\left(1-\text{FNR}\right)p -(1-p)\text{FPR} \right]^{n-1}\mathcal{P}(p)\,dp}{\int_0^1 \left[1-\left(1-\text{FNR}\right)p-(1-p)\text{FPR}\right]^{n-1}\mathcal{P}(p)\,dp}.
\end{split}
\end{equation}
We now show that $\mathbb{E}[\mathcal{R}_{\text{ReD}}]=\int_0^1 \left(1-\text{FNR}\right)p\mathcal{P}(p)\,dp=\pi_{1}>\pi_n$.
We set $\Omega_n(p)=\left[1-\left(1-\text{FNR}\right)p -(1-p)\text{FPR} \right]^{n-1}$ and observe that

\begin{equation}
    \pi_1-\pi_n=\frac{-(1-\text{FNR})\operatorname{Cov}\left(p,\Omega_n(p)\right)}{\mathbb{E}[\Omega_n(p)]}>0
\end{equation}
where the final inequality follows from $p$ being monotonically increasing and hence negatively correlated with $\Omega_n(p)>0$, which is monotonically decreasing on the unit interval. Because we have not assumed anything on $\mathcal{P}(p)$, this argument can be done with the conditional distribution after a single previous attempt, thus proving that $\pi_2>\pi_n\quad\forall n>2$. Therefore, by induction, the series $\{\pi_n\}_{n=1}^\infty$ is monotonically decreasing.

The mean reward per attempt on a generic strategy is thus
\begin{equation}
    \mathbb{E}[\mathcal{R}]=\sum_{n=1}^\infty \pi_n\Pr(n)<\sum_{n=1}^\infty \pi_1\Pr(n)=\pi_1=(1-\text{FNR})\mathbb{E}[p]=\mathbb{E}[\mathcal{R}_{\text{ReD}}],
\end{equation}
where $\Pr(n)$ is the fraction of time that the model spends answering problems for the $n$-th attempt. This proves that the expected reward of the ReD strategy is larger than that of any other strategy.

For a finite data set, ReD acts as a greedy algorithm, as can be seen from the monotonicity of $\pi_n$. It always attempts to answer the problem with the least number of failures, i.e., lowest $n$, and therefore highest expected reward.

Since $\pi_1>\pi_n$ for all $n>1$, at every attempt $t$ the per-attempt success probability of ReD exceeds that of any competing multiple-independent-attempt strategy. By linearity of expectation, summing over attempts $1,\dots,t$ yields $\mathrm{coverage@cost}_{\mathrm{ReD}}(t)\geq \mathrm{coverage@cost}_{S}(t)$ for every budget $t$, establishing the consequence stated after \cref{eq:imperfect}. 

\section{Comparison to other baselines}
\label[appendix]{app:baselines}

\subsection{Continuous reflection}
\label[appendix]{app:Continuous_reflection}
Continuous reflection (CR) is a simple self-correction strategy. If a problem is answered incorrectly, on the next attempt, we append an addition to the prompt that informs the model that the problem was answered incorrectly and provides the previous answer. This process is repeated at most five times for each problem in the GSM8K dataset.
The results below are for GSM8K.
\begin{table}[H]
    \centering
    \caption{Comparison with stronger allocation baselines: ReD versus CR for GSM8K ($N=1{,}319$, 1,000 realizations; 8b = llama-3.1-8b, 70b = llama-3.3-70b))}
    \begin{tabular}{llll}
        Strategy & $1 \times N$ & $2 \times N$ & $3 \times N$ \\
        \hline STC (8b) & $33.6 \%$ & $66.7 \%$ & $98.3 \%$ \\
        \hline CR (8b) & $32.1 \%$ & $63.6 \%$ & $95.3 \%$ \\
        \hline \textbf{ReD (8b)} & $\mathbf{8 1 . 3 \%}$ & $\mathbf{9 8 . 1 \%}$ & $\mathbf{9 9 . 0 \%}$ \\
        \hline STC (70b) & $25.1 \%$ & $49.6 \%$ & $74.4 \%$ \\
        \hline CR (70b) & $39.9 \%$ & $78.9 \%$ & $97.6 \%$ \\
        \hline \textbf{ReD (70b)} & $\mathbf{9 3 . 2 \%}$ & $\mathbf{9 7 . 0 \%}$ & $97.3 \%$
    \end{tabular}
    \label{tab:CR}
\end{table}
For the 70b model, CR outperforms STC at $1 \times N$ and $2 \times N$, but ReD is better. For 8b, CR performs similarly to STC. Smaller models tend to lack the ability to self-correct without substantial feedback, so CR buys little. Crucially, ReD dominates by a wide margin at every budget.

To generate ReD trajectories with self-correction, we performed the following experiment. In every ReD round, all problems that were not discarded in the previous round are asked again, and we append to the prompt a note saying they failed and give the previous answer. This strategy (Refl-ReD) did not substantially improve the performance of ReD for the GSM8K dataset, for both models, llama-3.1-8b and llama-3.3-70b, and 1000 realizations.
\begin{table}[H]
    \centering
    \caption{ReD beyond independent repeated sampling.}
    \begin{tabular}{llll} 
        Strategy & $1 \times N$ & $2 \times N$ & $3 \times N$ \\
        \hline Refl-ReD (8b) & $81.4 \%$ & $94.2 \%$ & $98.3 \%$ \\
        \hline ReD (8b) & $81.3 \%$ & $98.2 \%$ & $99.0 \%$ \\
        \hline Refl-ReD (70b) & $93.3 \%$ & $97.5 \%$ & $97.6 \%$ \\
        \hline ReD (70b) & $93.2 \%$ & $97.0 \%$ & $97.4 \%$
    \end{tabular}
    \label{tab:Refl-ReD}
\end{table}

\subsection{\texorpdfstring{$\tau$}{tau}-sweep on HumanEval}
\label[appendix]{app:tau_sweep}

We ran a $\tau$-sweep on HumanEval ($N=164$, both models, 200 realizations): ReD with $\tau$ attempts per problem per round before discarding, for $\tau \in \{1, 2, 4, 8\}$. $\tau=1$ is standard ReD and is optimal by \cref{eq:OptimalityReD}.

\begin{table}[H]
    \centering
    \caption{$\tau$-sweep on HumanEval: coverage at fixed total-attempt budgets (200 realizations). Bold = best at each budget.}
    \begin{tabular}{lllll}
        Strategy & $1{\times}N$ & $2{\times}N$ & $3{\times}N$ & $5{\times}N$ \\
        \hline \multicolumn{5}{l}{\textit{llama-3.1-8b}} \\
        \hline ReD ($\tau{=}1$) & $\mathbf{63.9\%}$ & $\mathbf{80.0\%}$ & $\mathbf{84.5\%}$ & $\mathbf{88.2\%}$ \\
        $\tau{=}2$ & $36.4\%$ & $73.1\%$ & $81.7\%$ & $87.3\%$ \\
        $\tau{=}4$ & $19.8\%$ & $39.5\%$ & $59.2\%$ & $83.6\%$ \\
        $\tau{=}8$ & $10.7\%$ & $20.8\%$ & $31.4\%$ & $52.2\%$ \\
        \hline \multicolumn{5}{l}{\textit{llama-3.3-70b}} \\
        \hline ReD ($\tau{=}1$) & $\mathbf{85.0\%}$ & $\mathbf{94.1\%}$ & $\mathbf{95.0\%}$ & $\mathbf{96.0\%}$ \\
        $\tau{=}2$ & $44.4\%$ & $88.8\%$ & $94.3\%$ & $95.7\%$ \\
        $\tau{=}4$ & $22.9\%$ & $45.8\%$ & $68.7\%$ & $94.4\%$ \\
        $\tau{=}8$ & $12.0\%$ & $23.4\%$ & $35.4\%$ & $58.7\%$ \\
    \end{tabular}
    \label{tab:tau_sweep}
\end{table}
ReD ($\tau=1$) dominates at every budget for both models, empirically confirming \cref{eq:OptimalityReD}.

\section{The influence of KV-cache on the GPU time with ReD}
\label[appendix]{app:kvcache}
In this Supplementary Note, we evaluate a conservative setting in which STC retains the KV cache whereas ReD does not, and show that ReD remains advantageous.

Following Zhong et al.~\cite{zhong2024distserve}, we estimate the overall request latency as the sum of the time-to-first-token (TTFT) plus time-per-output-token (TPOT) times the number of generated tokens. We also define per-token rates $L_p=\text{TTFT}/N_{\text{in}}$ (prefill) and $L_d=\text{TPOT}$ (decode), where $N_{\text{in}}$ and $N_{\text{out}}^{(j)}$ are the number of input tokens and output tokens on attempt $j$, respectively. We model GPU time per problem as $T(k)$, where $k$ is the number of attempts:
\begin{itemize}
    \item \textbf{STC} (prefill paid once, KV cache retained):
    \begin{equation*}
        T_{\text{STC}}(k)=N_{\text{in}}  L_p+\sum_{j=1}^k N_{\text{out}}^{(j)}  L_d
    \end{equation*}

    \item \textbf{Strict ReD} (full prefill on every attempt, zero cache retention):
    \begin{equation*}
        T_{\text{ReD}}(k)=\sum_{j=1}^k \left(N_{\text{in}}  L_p+ N_{\text{out}}^{(j)}  L_d\right)
    \end{equation*}
\end{itemize}

We set $L_p = 5$ ms/token and $L_d = 20$ ms/token ($L_d/L_p = 4$). We note that this ratio is a substantial underestimate: on H100 hardware running llama-3.1-8b, the E2E Networks inference benchmark~\cite{e2e_inference_benchmarks} reports TTFT $\approx$ 29ms for 128-token inputs ($L_p \approx 0.23$ ms/token) and TPOT $\approx 22$ ms/token, giving $L_d/L_p \approx 100$. A higher ratio means prefill is relatively cheaper, so the advantage of STC from caching a single prefill per problem shrinks. Our simulation therefore \textit{overestimates} the KV-cache benefit of STC and is conservative for ReD.

The results below are for GSM8K.
\begin{table}[H]
    \centering
    \caption{Hardware latency estimate for GSM8K ($N = 1{,}319$, 1,000 realizations, 8b=llama-3.1-8b, 70b=llama-3.3-70b)}
    \begin{tabular}{llll}
        GPU-s & Method & $\mathbf{8 b}$ & $\mathbf{7 0 b}$ \\
        \hline 2,219 & STC & $10.1 \%$ & $9.5 \%$ \\
        \hline & ReD & $\mathbf{2 6 . 0} \%$ & $\mathbf{3 4 . 0} \%$ \\
        \hline 11,516 & STC & $50.0 \%$ & $47.2 \%$ \\
        \hline & ReD & $\mathbf{9 5 . 6} \%$ & $\mathbf{9 6 . 9} \%$ \\
        \hline 21,952 & STC & $94.5 \%$ & $89.6 \%$ \\
        \hline & ReD & $\mathbf{9 8 . 8} \%$ & $\mathbf{9 7 . 4} \%$
    \end{tabular}
    \label{tab:Hardware_latency}
\end{table}

Even in the worst-case KV-cache scenario (zero retention), ReD outperforms STC at all GPU-second budgets tested. ReD's advantage compensates for the prefill overhead incurred on every context switch in this case.

\section{Reasoning effort}
\label[appendix]{app:reasoning}
Of the three models in \cref{sec:experiments}, only gpt-oss-20b is a reasoning model. We ran it at low and at medium reasoning effort (medium is the default, used in \cref{sec:experiments}) on HumanEval and on the 500 MMLU-Pro questions of \cref{app:benchmarks}. ReD saves attempts and output tokens at both settings, and its saving in output tokens grows with the reasoning effort: from about 2 to about 4 on HumanEval, and from about 4 to about 6 on MMLU-Pro (\cref{tab:effort_savings}). This is because harder problems need more attempts, and at higher effort also longer ones, and most of the attempts that ReD saves are on hard problems.

\begin{table}[H]
    \centering
    \caption{Saving of ReD for gpt-oss-20b at two reasoning efforts (200 realizations). Saving: the budget that STC needs to reach a target coverage ($25\%$, $50\%$ or $75\%$ of $N$), divided by that of ReD, in attempts or in output tokens.}
    \begin{tabular}{llcccccc}
        Benchmark & Effort & \multicolumn{3}{c}{Attempts} & \multicolumn{3}{c}{Output tokens} \\
         & & $25\%$ & $50\%$ & $75\%$ & $25\%$ & $50\%$ & $75\%$ \\
        \hline HumanEval & low & 1.6 & 1.7 & 1.7 & 2.1 & 2.2 & 2.2 \\
        \hline & medium & 2.4 & 2.4 & 2.3 & 4.1 & 4.0 & 3.9 \\
        \hline MMLU-Pro & low & 4.1 & 4.1 & 3.7 & 4.1 & 4.1 & 3.6 \\
        \hline & medium & 4.0 & 4.0 & 3.8 & 6.2 & 6.3 & 5.5
    \end{tabular}
    \label{tab:effort_savings}
\end{table}

\section{Evaluating the power-law exponent for llama-3.1-8b-instant}
\label[appendix]{sec:SI-experiments}

\begin{figure}[h]
    \centering
    \includegraphics[width=0.75\linewidth]{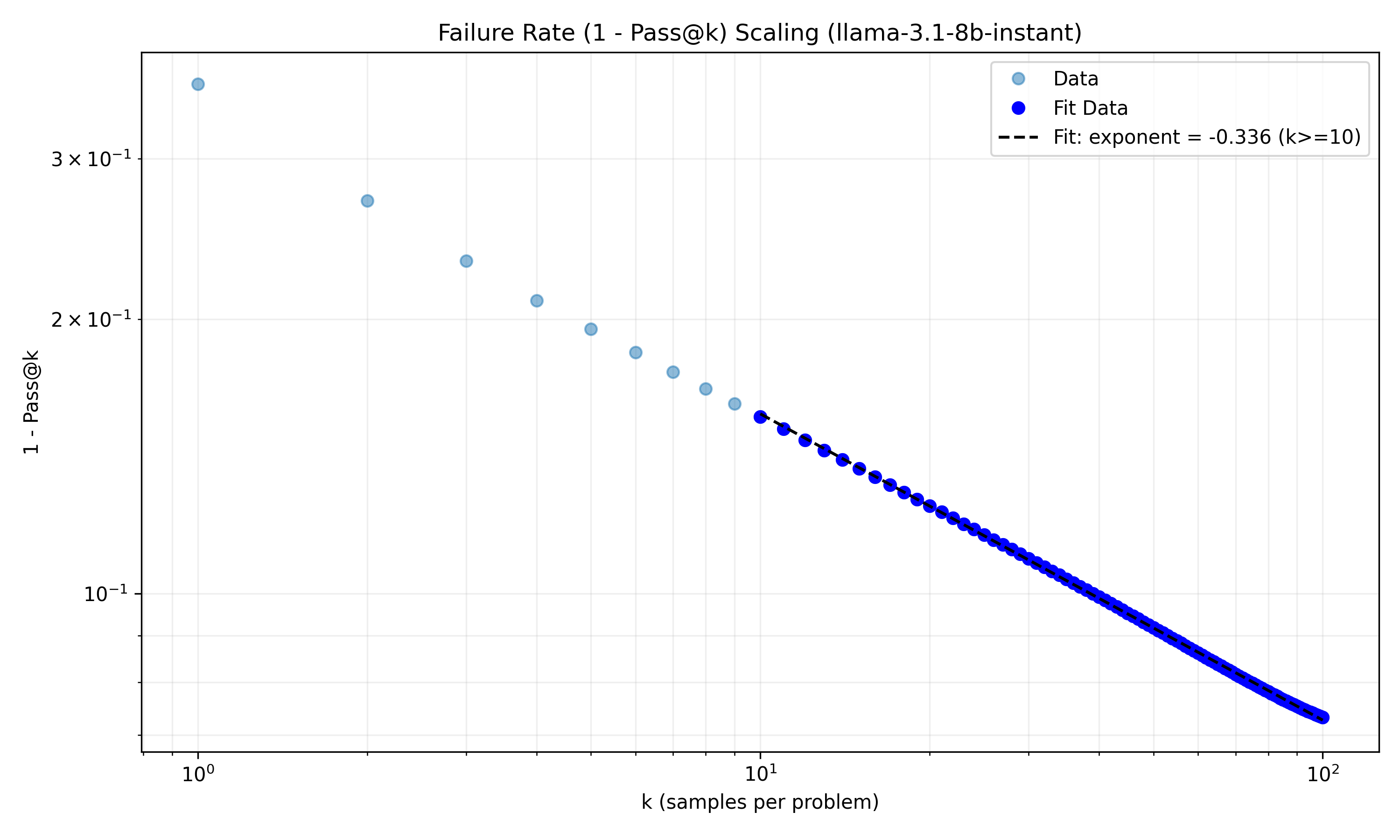}
    \caption{A linear fit of $\log(1-\text{pass@}k)$, of the llama-3.1-8b-instant, versus $\log k$. For $k\gg1$ indeed $(1-\text{pass@}k)\propto k^{-\alpha}$, and the obtained power-law exponent is $\alpha=0.34$.}
    \label[suppfigure]{fig:loglog_alpha}
\end{figure}

\end{appendices}

\end{document}